\documentclass[letterpaper]{article}

\usepackage[utf8]{inputenc}
\usepackage{amsmath,amscd,amssymb,amsthm}
\usepackage{verbatim}
\usepackage{thmtools}
\usepackage{thm-restate}

\usepackage[margin=1in]{geometry}

\usepackage[framemethod=TikZ]{mdframed}
\mdfdefinestyle{MyFrame}{%
    linecolor=black,
    outerlinewidth=2pt,
    roundcorner=20pt,
    innerrightmargin=20pt,
    innerleftmargin=20pt,
    backgroundcolor=white}
\usepackage[numbers]{natbib}
\usepackage{hyperref}

\newcommand{\R}{\mathbb{R}}
\newcommand{\E}{\mathop{\mathbb{E}}}
\newcommand{\argmin}{\mathop{\text{argmin}}}

\newcommand{\w}{u}

\newcommand{\ol}{\mathcal{A}}
\newcommand{\bol}{\mathcal{B}}

\usepackage{graphicx}

\usepackage{algorithm}
\usepackage{algorithmic}

\title{Combining Online Learning Guarantees}

\author{
  \textbf{Ashok Cutkosky}\\Google\\\texttt{ashok@cutkosky.com}
}
\date{}
\usepackage{times}

\begin{document}

\maketitle

\begin{abstract}%
  We show how to take any two parameter-free online learning algorithms with different regret guarantees and obtain a single algorithm whose regret is the minimum of the two base algorithms. Our method is embarrassingly simple: just add the iterates. This trick can generate efficient algorithms that adapt to many norms simultaneously, as well as providing diagonal-style algorithms that still maintain dimension-free guarantees. We then proceed to show how a variant on this idea yields a black-box procedure for generating \emph{optimistic} online learning algorithms. This yields the first optimistic regret guarantees in the unconstrained setting and generically increases adaptivity. Further, our optimistic algorithms are guaranteed to do no worse than their non-optimistic counterparts regardless of the quality of the optimistic estimates provided to the algorithm.
\end{abstract}

\section{Online Learning}\label{sec:online}

We consider the classic online learning problem with linear losses \citep{zinkevich2003online, shalev2011online, mcmahan2014survey}, sometimes called online linear optimization. Online learning is a game in which for each of $T$ rounds, the learning algorithm outputs some vector $w_t$ in some convex domain $W$, and then the environment reveals a vector $g_t$ and the algorithm suffers loss $\langle g_t, w_t\rangle$. The objective is to minimize the regret, which is the total loss relative to some benchmark point $\w$:
\[
R_T(\w):=\sum_{t=1}^T  \langle g_t, w_t-\w\rangle
\]
Although this formulation appears to only apply to a simple linear environment, algorithms that guarantee low regret can actually be automatically applied to general stochastic convex optimization problems found throughout machine learning \citep{cesa2004generalization}, and so many of the popular optimization algorithms in use today (e.g. \citep{duchi10adagrad, ross2013normalized}) are in fact online linear optimization algorithms.

Our first goal is to provide a ``meta-algorithm'' that combines online learning algorithms in a black-box manner to obtain an algorithm that achieves the best properties of the individual algorithms. Our technique applies to any algorithm that guarantees $R_T(0)$ is bounded by a constant, notably including the ``parameter-free'' algorithms that obtain regret bounds of the form $R_T(\w)=\tilde O(\|\w\|\sqrt{T})$ without knowledge of $\|\w\|$. There are already a number of such algorithms which guarantee regret bounds adapting to different characteristics of the sequence $g_t$ or comparison point $\w$ \citep{orabona2014simultaneous,foster2015adaptive, orabona2016coin,cutkosky2018black, foster2018online}. Our meta-algorithm frees the user from having to choose which algorithm is best for the task at hand.

Next, we develop a variation of this algorithm-combining technique that yields \emph{optimistic} regret guarantees. In optimistic online learning, the algorithm is provided with a ``hint'' $h_t$ that is some estimate of $g_t$ \emph{before} deciding on the prediction $w_t$. The goal is to use $h_t$ in such a way that the regret is very small when $h_t$ is a good estimate of $g_t$ \citep{hazan2010extracting, rakhlin2013online, chiang2012online, mohri2016accelerating}. A classic optimistic regret bound when $W$ has diameter $D=\sup_{x,y\in W} \|x-y\|$ is:
\[
R_T(\w) \le O\left(D\sqrt{\sum_{t=1}^T \|g_t-h_t\|^2}\right)
\]
Our approach is a reduction that takes an algorithm obtaining regret $R_T(\w)\le B(\w)\sqrt{\sum_{t=1}^T \|g_t\|_\star^2}$ for some arbitrary function $B$ and returns an algorithm obtaining regret
\[
R_T(\w) \le O\left[B(\w)\min\left(\sqrt{\sum_{t=1}^T \|g_t\|^2},\sqrt{\sum_{t=1}^T \|g_t-h_t\|^2}\right)\right]
\]
This improves on prior results in several ways. First, our algorithm is a generic reduction, and so can be applied to make optimistic versions of any new algorithms that may yet be invented. Second, it allows us to construct the first parameter-free optimistic algorithm (e.g. unbounded $W$, $B(\w)=\tilde O(\|\w\|)$). Third, when $W$ is unconstrained, we can improve our results to replace $\sum \|g_t-h_t\|^2$ with $\max(\sum \|g_t-h_t\|^2-\|h_t\|^2, 1)$. Finally, our optimistic algorithm is ``safe'' in the sense that even if the hints $h_t$ are very bad we still do no worse than the original algorithm.

This paper is organized as follows. First, we introduce our technique for combining online learning guarantees (Section \ref{sec:combine}), and provide an efficient algorithm that adapts to many norms at the same time as a simple example of the technique in action. Next, we apply this technique to generate optimistic algorithms in unconstrained domains (Section \ref{sec:optimism}) and see a generic improvement in adaptivity over prior optimistic algorithms while also maintaining good performance in the face of poor-quality $h_t$. We then proceed to adapt our optimistic algorithm to constrained domains (Section \ref{sec:constrained}), matching prior bounds while again being robust to bad $h_t$. Finally, we demonstrate how to take advantage of \emph{multiple} sequences of hints (Section \ref{sec:manyhints}), obtaining an optimistic guarantee that matches the performance on the best sequence of hints in hindsight. We conclude with a simple trick showing how to compete with the best \emph{fixed} hint (Section \ref{sec:fixed}).

% We then further expand on our best-of-all-worlds theme to provide an algorithm that can take advantage of \emph{multiple} sequences of hints and obtain a regret that performs as well as the best individual sequence (Section \ref{sec:manyhints}). Finally, we conclude with a simple trick for setting the hints $h_t$ that guarantees regret reminiscent of \citep{hazan2010extracting} in the unconstrained setting:
% \[
% R_T(\w) \le O\left[\inf_{\overline{g}} B(\w) \sqrt{\log(T) + \sum_{t=1}^T \|g_t- \overline{g}\|_\star^2}\right]
% \]
% This bound can also be combined with externally supplied hints to obtain the best bound in hindsight. Further, we remark that this bound provides a simple proof of an empirical Bernstein bound \citep{maurer2009empirical} in Hilbert spaces (Section \ref{sec:fixed}).

\subsection{Definitions and Notation}
Throughout this paper we assume $W$ is a convex subset of a real Hilbert space. Given a norm $\|\cdot\|$, we write $\|\cdot\|_\star$ to indicate the dual norm $\|g\|_\star=\sup_{\|x\|\le 1}\langle g, x\rangle$. We always use $\|\cdot\|$ to indicate the Hilbert space norm unless otherwise stated, so that $\|\cdot\|=\|\cdot\|_\star$ by the standard identification of a Hilbert space with its dual. Given a convex function $f$, we write $x\in\partial f(y)$ to indicate that $x$ is a subgradient of $f$ at $y$. We interchangeably refer to $g_t$ as losses and gradients. We will often assume the $g_t$ are bounded $\|g_t\|\le 1$ for all $t$, which will be stated explicitly in the hypotheses of the relevant results. As usual, $e$ indicates the base of the natural logarithm.

\section{Combining Parameter-Free Algorithms}\label{sec:combine}
In this section we provide our technique for combining incomparable regret guarantees.
Our technique is most effective on algorithms that ensure $R_T(0)\le \epsilon$ for some (usually user-specified) $\epsilon$. There has been much recent work on this style of algorithm \citep{mcmahan2012no, orabona2013dimension, mcmahan2014unconstrained, foster2015adaptive, orabona2017training, cutkosky2018black}, yielding so-called parameter-free algorithms that achieve optimal or near-optimal regret guarantees up to log factors. These works provide various improvements in adaptivity to the norm of $\w$ or the gradients $g_t$. However, there is no one uniformly-dominant adaptive guarantee. As a simple example, under the assumption $\|g_t\|_\star \le 1$ for all $t$, recently \citep{cutkosky2018black} provided algorithms that obtain
\begin{align}
    R_T(\w) &\le \tilde O\left[\epsilon + \frac{\|\w\|}{\sqrt{\lambda}}\max\left(\log\left(\frac{\|\w\|T}{\epsilon}\right),\ \sqrt{\sum_{t=1}^T \|g_t\|_\star^2\log\left(\frac{\|\w\|T}{\epsilon}\right)}\right)\right]\label{eqn:normbound}
\end{align}
for any norm fixed norm $\|\cdot\|$ such that $\|\cdot\|^2$ is $\lambda$-strongly convex with respect to the norm $\|\cdot\|$.

% \begin{align}
%     R_T(\w) &\le O\left[\epsilon + \|\w\|_2\max\left(8\log\left(\frac{8\|\w\|_2\left(1+\sum_{t=1}^T \|g_t\|_2^2\right)^{4.5}}{\epsilon}\right),\ \right.\right.\nonumber\\
%     &\left.\left.\quad\quad\quad\quad 2\sqrt{\sum_{t=1}^T \|g_t\|_2^2\log\left(\frac{5\|\w\|^2}{\epsilon^2}\left(8\sum_{t=1}^T \|g_t\|_2^2 +2\right)^{10} + 1\right)}\right)\right.\nonumber\\
%     &\left.\quad\quad\quad\quad +\|\w\|_2 \sqrt{\sum_{t=1}^T \|g_t\|_2^2}\right]\label{eqn:l2regret}
% \end{align}
% in $\R^d$. 
Further, by running a single 1-dimensional copy of this algorithm in each coordinate of a $d$-dimensional problem and rescaling $\epsilon$ to $\epsilon/d$, we can obtain the regret:
\begin{align}
    R_T(\w) &\le \tilde O\left[\epsilon + \sum_{t=1}^T |\w_i|\max\left(\log\left(\frac{d|\w_i|T}{\epsilon}\right),\ \sqrt{\sum_{t=1}^T |g_{t,i}|^2\log\left(\frac{d|\w_i|T}{\epsilon}\right)}\right)\right]\label{eqn:percoordinate}
\end{align}

% \begin{align*}
%     R_T(\w) &\le \sum_{t=1}^T \sum_{i=1}^d w_{t,i} g_{t,i} - \w_i g_{t,i}\\
%     &\le O\left[\epsilon + \sum_{i=1}^d |\w_i|\max\left(8\log\left(\frac{8|\w_i|\left(d+d\sum_{t=1}^T g_{t,i}^2\right)^{4.5}}{\epsilon}\right),\ \right.\right.\\
%     &\left.\left.\quad\quad\quad\quad 2\sqrt{\sum_{t=1}^T g_{t,i}^2\log\left(\frac{d^25|\w_i|^2}{\epsilon^2}\left(8\sum_{t=1}^T g_{t,i}^2 +2\right)^{10} + 1\right)}\right)\right.\\
%     &\left.\quad\quad\quad\quad +|\w_i| \sqrt{\sum_{t=1}^T g_{t,i}^2}\right]
% \end{align*}

    % &\le O\left[\epsilon + \|\w\|_2\max\left(8\log\left(\frac{8\|\w\|_2\left(1+\sum_{t=1}^T \|g_t\|_2^2\right)^{4.5}}{\epsilon}\right),\ \right.\right.\\
    % &\left.\left.\quad\quad\quad\quad 2\sqrt{\sum_{t=1}^T \|g_t\|_2^2\log\left(\frac{5\|\w\|^2}{\epsilon^2}\left(8\sum_{t=1}^T \|g_t\|_2^2 +2\right)^{10} + 1\right)}\right)\right.\\
    % &\left.\quad\quad\quad\quad +\|\w\| \sqrt{\sum_{t=1}^T \|g_t\|_2^2}\right]

% Further, \citep{cutkosky2018black} provides algorithms that obtain regret
% \begin{align*}
%     R_T(\w) &\le O\left[\epsilon + \|\w\|\max\left(8\log\left(\frac{8\|\w\|\left(1+\sum_{t=1}^T \|g_t\|_\star^2\right)^{4.5}}{\epsilon}\right),\ \right.\right.\\
%     &\left.\left.\quad\quad\quad\quad 2\sqrt{\sum_{t=1}^T \|g_t\|_\star^2\log\left(\frac{5\|\w\|^2}{\epsilon^2}\left(8\sum_{t=1}^T \|g_t\|_\star^2 +2\right)^{10} + 1\right)}\right)\right.\\
%     &\left.\quad\quad\quad\quad \frac{\|\w\|}{\sqrt{\lambda}} \sqrt{\sum_{t=1}^T \|g_t\|_\star^2}\right]
% \end{align*}
% whenever $\|\cdot\|^2$ is a $\lambda$-strongly convex with respect to the norm $\|\cdot\|$.

These regret guarantees are optimal (up to log factors) and also incomparable a priori. Depending on the gradients $g_t$ and the benchmark $\w$, it may be best to use the per-coordinate algorithm or it may be best to use some particular norm. Thus the ``ultimate adaptive algorithm'' would be able to achieve the best of all these bounds \emph{in hindsight}. One approach might be to run all of these optimizers in parallel and use some kind of expert algorithm to choose the best one. This is essentially the approach taken by \citep{foster2017parameter}. However, the regret of such a scheme would likely scale with the \emph{maximum} loss experienced by the best algorithm, which may be extremely pessimistic. Alternatively, one might consider the simpler strategy of simply averaging the predictions of the base algorithms. Unfortunately, now the regret is the \emph{average} of the individual regrets, which is still not good enough. Instead, we propose an even simpler scheme: just add the predictions. Rather surprisingly, this strategy works so long as each base algorithm guarantees $R_T(0)$ sufficiently small. Specifically, we have the following easy Theorem:

\begin{restatable}{Theorem}{thmcombine}\label{thm:combine}
Suppose $W$ is a Hilbert space. Let $\ol$ and $\bol$ be two online linear optimization algorithms that guarantee regret $R^{\ol}_T(\w)$ and $R^{\bol}_T(\w)$ respectively. Let $w^{\ol}_t$ and $w^{\bol}_t$ be their respective predictions on the loss sequence $g_1,\dots, g_T$. Let $w_t=w^{\ol}_t+w^{\bol}_t$. Then we have
\begin{align*}
    R_T(\w)=\sum_{t=1}^T \langle g_t, w_t-\w\rangle \le \inf_{x+y=\w} R^{\ol}_T(x) + R^{\bol}_T(y)
\end{align*}
In particular, if $R^{\ol}_T(0)\le \epsilon$ and $R^{\bol}_T(0)\le \epsilon$, we have
\begin{align*}
    R_T(\w)&\le \epsilon+\min(R^{\ol}_T(\w), R^{\bol}_T(\w))
\end{align*}
\end{restatable}
\begin{proof}
The proof is one line:
\begin{align*}
    \sum_{t=1}^T \langle g_t, w_t-\w\rangle=\sum_{t=1}^T \langle g_t, w^{\ol}_t-x\rangle+\langle g_t, w^{\bol}_t-y\rangle\le R^{\ol}_T(x) + R^{\bol}_T(y)
\end{align*}
\end{proof}
With this strategy it is clear that we can combine any $k$ algorithms and obtain only an additive penalty of $(k-1)\epsilon$ over the best of their regret bounds. Since parameter-free algorithms with guarantees like (\ref{eqn:normbound}) and (\ref{eqn:percoordinate}) depend on $\log(1/\epsilon)$, we can replace $\epsilon$ with $\epsilon/k$ to increase the regret by a factor of $\log(k)$ in exchange for guaranteeing only $\epsilon$ regret at 0. Finally, we note that our assumption that $W$ is an entire Hilbert space can usually be removed using the unconstrained-to-constrained reduction of \citep{cutkosky2018black}.

We can gain some more insight into why a result such as Theorem \ref{thm:combine} should be expected to exist by appealing to the equivalence between regret bounds and concentration inequalities outlined by \citep{rakhlin2015equivalence}. Roughly speaking, this result says that a regret bound of $R_T$ implies that sums of mean-zero random variables concentrate about their mean with a radius of roughly $R_T$ - and vice versa. Therefore one should be able to convert a concentration bound into an online learning algorithm, although the conversion may be very computationally taxing. There is already an extremely popular technique for combining concentration inequalities - the union bound - so there should be a corresponding way to combine regret bounds. In this way we can view Theorem \ref{thm:combine} as providing an extremely efficient online learning analog to the union bound.

One immediate application of Theorem \ref{thm:combine} is to combine an algorithm that obtains the bound (\ref{eqn:percoordinate}) with one that obtains the bound (\ref{eqn:normbound}) where $\|\cdot\|=\|\cdot\|_2$. This yields an algorithm that simultaneously enjoys a ``dimension-free'' bound with respect to $\|\cdot\|_2$ while also reaping the benefits of per-coordinate updates when the gradients $g_t$ or comparison point $\w$ are sparse.

A second application is to adapt to many norms simultaneously. For example, in the following Theorem we construct an algorithm that adapts to any $p$-norm for $p\in[1,2]$. The strategy is simple: first, we show that by selecting a discrete grid of $\log(d)$ different $p_i$, we can ensure $\|x\|_p$ is within a constant of $\|x\|_{p_i}$ for some $i$ for any $p\in[1,2]$ (Lemma \ref{thm:pgrid}). Then we observe that $\|\cdot\|_p$ is $p-1$-strongly-convex with respect to itself, so that combining the guarantee (\ref{eqn:normbound}) with our algorithm-combining strategy immediately yields the desired results (Theorem \ref{thm:allpnorm}).

\begin{restatable}{Theorem}{thmallpnorm}\label{thm:allpnorm}
Suppose $g_t\in \R^d$ satisfies $\|g_t\|_2 \le 1$ for all $t$. Then there exists an online algorithm that runs in time $O(d\log(d))$ per update that obtains regret
\begin{align*}
R_T(\w)&\le \tilde O\left(\epsilon\log(d)+\inf_{p\in [1,2]} \frac{\|\w\|_p}{\sqrt{p-1}}\sqrt{\sum_{t=1}^T \|g_t\|_q^2}\right)
\end{align*}
where for any $p$, $q$ is such that $\frac{1}{p}+\frac{1}{q}=1$.
\end{restatable}
\begin{proof}
Consider $q_0=2$ and $\frac{1}{q_i}=\frac{1}{q_{i-1}}-\frac{1}{\log(d)}$ for all $i\le \log(d)/2$ and $p_i$ given by $\frac{1}{p_i}+\frac{1}{q_i}=1$. Note that there are $O(\log(d))$ different indices $i$. Then by Lemma \ref{thm:pgrid}, for any $p\in[1,2]$, there is some $p_i\ge p$ such that $\|\w\|_{p_i}\le \|\w\|_p$ and $\|g\|_{q_i}\le e\|g\|_q$ for all $g$. Recall that $\|\cdot\|_{p_i}^2$ is $p_i-1$-strongly convex with respect to $\|\cdot\|_{p_i}$. Then consider running one algorithm for each $p_i$ that guarantees regret
\begin{align*}
    R_T(\w)&\le \epsilon + \tilde O\left(\frac{\|\w\|_{p_i}}{\sqrt{p_i-1}}\sqrt{\sum_{t=1}^T \|g_t\|_{q_i}^2}\right)
\end{align*}
where $\frac{1}{p_i}+\frac{1}{q_i}=1$.
Note that this is possible because $\|g_t\|_{q_i}\le \|g_t\|_2\le 1$ for all $t$. Then by combining all $O(\log(d))$ of these algorithms using Theorem \ref{thm:combine} we obtain the stated result.
\end{proof}
Similar bounds have been shown in previous work: \citep{foster2017parameter} achieved a similar bound using an expert algorithm to combine the base algorithms. However, the expert algorithm dominates the runtime and leads to both $O(T)$ time per update, and also to loss of adaptivity to the sum of the squared norms of the gradients. Also, \citep{cutkosky2018black} provides an algorithm that adapts to any sequence of norms simultaneously, but their algorithm requires $O(d^2)$ time per update and incurs an extra $\sqrt{d}$ factor in the regret bound. In contrast, the algorithm presented above is simple, adaptive, and efficient.

This best-of-all-worlds technique has powerful applications beyond simply combining existing regret guarantees. In particular, it enables us to combine algorithms that \emph{do not guarantee sublinear regret} with algorithms that do have reasonable worst-case regret guarantees. This enables us to generate algorithms that perform well all the time (because they do no worse than the algorithm with a worst-case guarantee), but may sometimes perform much better because sometimes the algorithm without a sublinear regret guarantee may ``get lucky'' and perform extremely well. In the following sections, we elaborate on this idea to develop \emph{optimistic} online algorithms.

\section{Optimism}\label{sec:optimism}
Now we turn our best-of-all-worlds strategy into an optimistic online learning algorithm. Specifically, we will provide a black-box reduction that converts any algorithm $\ol$ that obtains regret
\begin{align*}
    R_T(\w) &\le B(\w)\sqrt{\sum_{t=1}^T \|g_t\|^2}
\end{align*}
into an optimistic algorithm obtaining regret
\begin{align*}
    R_T(\w)&\le B(\w)\sqrt{\sum_{t=1}^T \|g_t-h_t\|^2}
\end{align*}
We first tackle the problem in the case that $W$ is an entire Hilbert space (no constraints), and then move to a constrained setting in Section \ref{sec:constrained}.

Our strategy uses a 1-dimensional parameter-free algorithm to take advantage of the hint $h_t$. Intuitively, if $h_t=g_t$ for all $t$, then playing $w_t=-yh_t$ for some sufficiently large positive constant $y$ will yield small regret. We can learn this constant $y$ on-the-fly by using a 1-dimensional online algorithm. Alternatively, if the hints are bad, then simply running $\ol$ will yield reasonably low regret (although perhaps not as low as in the former case). We combine these two approaches using the technique of Theorem \ref{thm:combine}, and then add some more detailed analysis to derive the optimistic regret guarantee. Importantly, this extra analysis allows us to dispense with the requirement that $\ol$ guarantees regret $\epsilon$ at the origin, so that we can make optimistic versions of essentially any adaptive online learning algorithm.

 \begin{algorithm}
   \caption{Optimistic Reduction}
   \label{alg:opt}
\begin{algorithmic}
   \STATE {\bfseries Input:} Online learning algorithm $\ol$ with domain $W$ and $\bol$ with domain $\R$.
   \FOR{$t=1$ {\bfseries to} $T$}
   \STATE Get $x_t$ from $\ol$ and $y_t$ from $\bol$.
   \STATE Get hint $h_t$.
   \STATE Play $w_t=x_t - y_th_t$, receive loss $g_t$.
   \STATE Send $g_t$ to $\ol$ as the $t$th loss.
   \STATE Send $-\langle g_t, h_t\rangle$ to $\bol$ as the $t$th loss.
   \ENDFOR
\end{algorithmic}
\end{algorithm}
\begin{restatable}{Theorem}{thmopt}\label{thm:opt}
Let $W$ be a Hilbert space. Suppose $\ol$ guarantees regret
\begin{align*}
    R^{\ol}_T(\w)\le A_T(\w)+B_T(\w)\sqrt{\sum_{t=1}^T \|g_t\|^2}
\end{align*}
on gradients $g_t$ and suppose $\bol$ guarantees regret
\begin{align*}
    R^{\bol}_T(\w)\le \epsilon+|\w|C\log(1+|\w|T^c/\epsilon)+|\w|D\sqrt{\sum_{t=1}^T z_t^2\log(1+|\w|T^c/\epsilon)}
\end{align*}
on gradients $z_t$ with $|z_t|\le 1$, where $A_T$ and $B_T$ are arbitrary non-negative functions and $C$, $c$ and $D$ and $\epsilon$ are arbitrary nonnegative constants. Finally, suppose $\|h_t\|\le 1$ and $\|g_t\|\le 1$ for all $t$.
Then Algorithm \ref{alg:opt} guarantees regret
\begin{align*}
R_T(\w)&\le B_T(u)\sqrt{\left[(2C+D^2)\log(e+B_T(u)T^c/\epsilon)+\sum_{t=1}^T \|h_t-g_t\|^2-\|h_t\|^2\right]_1}\\
    &\quad\quad\quad+DB_T(u)\sqrt{\log(e+B_T(u)T^c/\epsilon)} + A_T(u)+\epsilon
\end{align*}
where $[X]_1$ denotes $\max(X,1)$. Further, Algorithm \ref{alg:opt} \emph{simultaneously} guarantees regret
\begin{align*}
    R_T(\w) &\le \epsilon+R^{\ol}_T(\w)
\end{align*}
\end{restatable}
Let us unpack this Theorem. If we remove all logarithmic factors, then the Theorem states that
\begin{align*}
    R_T(\w) &\le \tilde O\left(B_T(u)\sqrt{\left[\sum_{t=1}^T \|h_t-g_t\|^2-\|h_t\|^2\right]_1}+A_T(u)+\epsilon\right)
\end{align*}
Next, we recall that \citep{cutkosky2018black} provides a 1-D algorithm that satisfies the conditions of Theorem \ref{thm:opt} for $\bol$ as well as an algorithm that satisfies the conditions for $\ol$ with $B_T(\w)=O(\|\w\|\sqrt{\log(\|\w\|T/\epsilon)})$ and $A_T(\w)=O(\|\w\|\log(\|\w\|T/\epsilon)+\epsilon)$\footnote{for example, consider the regret bound (\ref{eqn:normbound}) with $\lambda=1$}. Thus using these algorithms we obtain:
\begin{align*}
    R_T(\w)&\le \tilde O\left(\epsilon + \|\w\|\sqrt{\left[\sum_{t=1}^T \|h_t-g_t\|^2-\|h_t\|^2\right]_1}\right)
\end{align*}
This is already somewhat better (modulo log factors) than the standard optimistic guarantee by virtue of the $-\|h_t\|^2$ terms. Further, this algorithm is \emph{unconstrained}, and to our knowledge is the first unconstrained algorithm to achieve this optimistic guarantee. Even more, the second part of the Theorem shows that we never do worse than the base algorithm $\ol$ \emph{regardless of the values of $h_t$}. This greatly robustifies optimistic online algorithms, as it allows the use of arbitrary hint sequences that may have absolutely no relationship with $g_t$ without harming the regret guarantees.

Now we provide the proof of Theorem \ref{thm:opt}. As sketched above, the main idea is that we are using Theorem \ref{thm:combine} to combine the regret of $\ol$ and $\bol$. By careful analysis of the regret of these two algorithms we can interpolate between the optimal scenario for $\bol$ (i.e. when the $g_t=h_t$ for all $t$), and the more general adversarial scenario.

\begin{proof}
We write the regret
\begin{align*}
    R_T(\w)&=\sum_{t=1}^T \langle g_t, w_t-\w\rangle\\
    &= \sum_{t=1}^T \langle g_t, x_t - \w\rangle - \langle g_t, h_t\rangle y_t\\
    &\le R^{\ol}_T(\w) + R^{\bol}_T(y) - \sum_{t=1}^T \langle g_t, h_t\rangle y
\end{align*}
Now we can actually immediately see the second part of the Theorem: just set $y=0$ and observe that $R^{\bol}_T(0)\le \epsilon$. With this out of the way, we continue to unpack our regret inequality:
\begin{align*}
    R_T(\w)&\le  A_T(\w)+B_T(\w)\sqrt{\sum_{t=1}^T \|g_t\|^2}+yD\sqrt{\sum_{t=1}^T \langle g_t, h_t\rangle^2\log(1+|y|T^c/\epsilon)}\\
    &\quad\quad+ \epsilon + yC\log (1+|y|T^c/\epsilon)-\sum_{t=1}^T \langle g_t, h_t\rangle y\\
    &\le A_T(\w) + (B_T(\w) + yD\sqrt{\log(1+|y|T^c/\epsilon)})\sqrt{\sum_{t=1}^T \|g_t\|^2} - \sum_{t=1}^T \langle g_t, h_t\rangle y\\
    &\quad\quad+ \epsilon +yC\log(1+|y|T^c/\epsilon)\\
\end{align*}
Where in the second line we used $\|h_t\|\le 1$.

Now consider the identity $-2\langle g_t, h_t\rangle=\|g_t-h_t\|^2-\|g_t\|^2-\|h_t\|^2$. Applying this yields
\begin{align*}
    R_T(\w)&\le yC\log(1+|y|T^c/\epsilon)+ (B_T(\w) + yD\sqrt{\log(1+|y|T^c/\epsilon)})\sqrt{\sum_{t=1}^T \|g_t\|^2} \\
    &\quad\quad\quad-\frac{y}{2}\sum_{t=1}^T\|g_t\|^2 + \frac{y}{2}\sum_{t=1}^T\|h_t-g_t\|^2-\|h_t\|^2+A_T(\w) + \epsilon \\
    &\le yC\log(1+|y|T^c/\epsilon)+ \frac{y}{2}\sum_{t=1}^T\|h_t-g_t\|^2-\|h_t\|^2\\
    &\quad\quad\quad+\sup_{X\ge 0}\left[(B_T(\w) + yD\sqrt{\log(1+|y|T^c/\epsilon)})\sqrt{X} -\frac{y}{2}X\right] +A_T(\w) + \epsilon  \\
    &\le  yC\log(1+yT^c/\epsilon) +\frac{y}{2}\sum_{t=1}^T \|h_t-g_t\|^2-\|h_t\|^2\\
    &\quad\quad+\frac{(B_T(\w) + yD\sqrt{\log(1+|y|T^c/\epsilon)})^2}{2y} + A_T(\w) + \epsilon
\end{align*}
where in the last line we have used the assumption $y\ge 0$.
Now we optimize $y$:
\begin{align*}
    R_T(\w)&\le\inf_{y\ge 0}\left[ yC\log(1+yT^c/\epsilon) +\frac{y}{2}\sum_{t=1}^T \|h_t-g_t\|^2-\|h_t\|^2\right.\\
    &\quad\quad\quad\left.+\frac{(B_T(\w) + yD\sqrt{\log(1+yT^c/\epsilon)})^2}{2y}\right] +A_T(\w) +  \epsilon \\
    &\le \inf_{y\ge 0}\left[ \frac{y}{2}\left((2C+D^2)\log(1+yT^c/\epsilon) +\sum_{t=1}^T \|h_t-g_t\|^2-\|h_t\|^2\right)+\frac{B_T(\w)^2}{2y}\right.\\
    &\quad\quad\quad\left.+DB_T(\w)\sqrt{\log(e+yT^c/\epsilon)}\right]+ A_T(\w) +\epsilon 
\end{align*}
This infimum is computed in Lemma \ref{thm:balancelogs}, yielding:
\begin{align*}
    R_T(\w)&\le B_T(u)\sqrt{\left[(2C+D^2)\log(e+B_T(u)T^c/\epsilon)+\sum_{t=1}^T \|h_t-g_t\|^2-\|h_t\|^2\right]_1}\\
    &\quad\quad\quad+DB_T(u)\sqrt{\log(e+B_T(u)T^c/\epsilon)} + A_T(u)+\epsilon
\end{align*}
as desired.
\end{proof}
\section{Constrained Optimism}\label{sec:constrained}
The reduction Algorithm \ref{alg:opt} requires an unconstrained domain in order to form the updates $x_t - y_th_t$. To move to the constrained setting, we use the unconstrained-to-constrained reduction from \citep{cutkosky2018black}. When used out-of-the-box, this reduction converts an unconstrained algorithm whose regret as a function of the gradients $g_t$ is $R_T(\w,g_t,\dots,g_T)$ into a constrained algorithm that obtains regret $2R_T(\w,\tilde g_t,\dots,\tilde g_t)$ where $\tilde g_t$ is a ``surrogate gradient'' with $\|\tilde g_t\|\le\|g_t\|$. Unfortunately, this is not quite good enough to maintain optimism as $\tilde g_t$ may be less similar to $h_t$ than $g_t$ was. However, by inspecting the internals of the reduction, we can remedy this issue.

Specifically, the reduction of \citep{cutkosky2018black} replaces the iterates $w_t$ of the unconstrained online learning algorithm with $\Pi(w_t)$ and the gradients $g_t$ with $\tilde g_t=\frac{g_t}{2}  + \frac{\|g_t\|}{2} \nabla S (w_t)$, where $\Pi(w)=\argmin_{w'\in W} \|w-w'\|$ and $S(w) = \|w-\Pi(x)\|=\inf_{w\in W}\|x-w\|$ (in Hilbert spaces, $\argmin_{w'\in W} \|w-w'\|$ is always a singleton). We therefore apply the same transformation to the hint $h_t$, replacing it with $\tilde h_t=\frac{h_t}{2}+\frac{\|h_t\|}{2}\nabla S(w_t)$. This strategy is actually somewhat more subtle than it appears because $w_t$ is a function of $\tilde h_t$: $w_t = x_t-y_t\tilde h_t$, and so the setting $\tilde h_t=\frac{h_t}{2}+\frac{\|h_t\|}{2}\nabla S(w_t)$ actually represents an equation that must be solved for the value of $\tilde h_t$. Fortunately, it turns out that this equation is not too difficult to solve, with the help of the following Lemma:

\begin{restatable}{Lemma}{thmhtilde}\label{thm:htilde}
Let $W$ be a convex domain in a Hilbert space $H$. Let $x\in H$, $y\in \R$, $h\in H$. Let $z\in \partial S(x-\frac{yh}{2})$. Suppose $\frac{y\|h\|z}{2} \le S(x-yh/2)$. Then $z\in \partial S(x-y\tilde h)$ for $\tilde h = \frac{h}{2} + \frac{\|h\|z}{2}$.  If instead $\frac{y\|h\|z}{2} < S(x-yh/2)$, then $az\in \partial S(x-y\tilde h)$ for $\tilde h = \frac{h}{2}+\frac{\|h\|az}{2}$ where $a= \frac{2S(x-\frac{yh}{2})}{y\|h\|}$.
\end{restatable}

Intuitively, this Lemma tells us that most of the time if we set $\tilde h_t = \frac{h}{2} + \frac{\|h\|z}{2}$ for $z=\nabla S(x_t-yh_t/2)$, we will have $\tilde h_t = \frac{h_t}{2}+\frac{\|h_t\|}{2}\nabla S(w_t)$, where $w_t = x_t-y_t\tilde h_t$. This suggests the reduction given by Algorithm \ref{alg:constrained} for constrained optimism.

\begin{algorithm}
   \caption{Optimism with Constraints}
   \label{alg:constrained}
\begin{algorithmic}
   \STATE {\bfseries Input:} Online learning algorithms $\ol$ with domain $W$ and $\bol$ with domain $\R$.
   \FOR{$t=1$ {\bfseries to} $T$}
   \STATE Get $x_t$ from $\ol$ and $y_{t,i}$ from $\bol$.
   \STATE Get hint $h_t$.
   \STATE Compute $z_t\in \partial S(x_t-y_t\frac{h_t}{2})$. 
   \IF{$y_t\|h_t\|/2> S(x+\frac{y_th_t}{2})$}
   \STATE Set $a = \frac{2S(x+\frac{y_th_t}{2})}{y_t\|h_t\|}$.
   \STATE Set $z_t=az$.
   \ENDIF
   \STATE Set $\tilde h_t = \frac{h_t}{2} + \frac{\|h_t\|z_t}{2}$
   \STATE Set $\tilde w_t=x_t - y_t\tilde h_t$.
   \STATE Play $w_t = \Pi(w_t)$, receive loss $g_t$.
   \STATE Set $\tilde g_t = \frac{g_t}{2} + \frac{z_t\|g_t\|}{2}$
   \STATE Send $\tilde g_t$ to $\ol$ as the $t$th loss.
   \STATE Send $-\langle \tilde g_t, \tilde h_t\rangle$ to $\bol$ as the $t$th loss.
   \ENDFOR
\end{algorithmic}
\end{algorithm}
\begin{restatable}{Theorem}{thmconstrained}\label{thm:constrained}
Under the same assumptions as Theorem \ref{thm:opt}, with the exception that $W$ is now a convex domain in a Hilbert space rather than necessarily the entire space.
% Suppose $\ol$ guarantees regret
% \begin{align*}
%     R^{\ol}_T(\w)\le A_T(\w)+B_T(\w)\sqrt{\sum_{t=1}^T \|g_t\|^2}
% \end{align*}
% on gradients $g_t$ and suppose $\bol$ guarantees regret
% \begin{align*}
%     R^{\bol}_T(\w)\le \epsilon+|\w|C\log(1+|\w|T^c/\epsilon)+|\w|D\sqrt{\sum_{t=1}^T z_t^2\log(1+|\w|T^c/\epsilon)}
% \end{align*}
% on gradients $z_t$ for $|z_t|\le 1$, where $A_T$ and $B_T$ are arbitrary non-negative functions and $C$, $c$ and $D$ and $\epsilon$ are arbitrary nonnegative constants. Further suppose $\|g_t\|\le 1$ and $\|h_t\|\le 1$ for all $t$.
Then Algorithm \ref{alg:constrained} guarantees regret
\begin{align*}
R_T(\w)&\le 2B_T(u)\sqrt{\left[(2C+D^2)\log(e+B_T(u)T^c/\epsilon)+\sum_{t=1}^T \| h_t-g_t\|^2\right]_1}\\
    &\quad\quad\quad+2DB_T(u)\sqrt{\log(e+B_T(u)T^c/\epsilon)} + 2A_T(u)+2\epsilon
\end{align*}
where $[X]_1$ denotes $\max(X,1)$. Further, Algorithm \ref{alg:opt} \emph{simultaneously} guarantees regret
\begin{align*}
    R_T(\w) &\le 2\epsilon+2A_T(\w)+2B_T(\w)\sqrt{\sum_{t=1}^T \|g_t\|^2}
\end{align*}
\end{restatable}
With this constrained algorithm in hand, we can take advantage of adaptive gradient descent algorithms \citep{duchi10adagrad,mcmahan2010adaptive} that obtain guarantees $R_T(\w)\le B\sqrt{2\sum_{t=1}^T \|g_t\|^2}$ where $B$ is the diameter of $W$. Using such an algorithm as $\ol$ and the same 1-dimensional algorithm for $\bol$ as in the discussion following Theorem \ref{thm:opt}, we obtain a regret of
\begin{align*}
    R_T(\w)&\le O\left(\epsilon + B\sqrt{2\sum_{t=1}^T \|g_t-h_t\|^2 + \log(BT/\epsilon)} + B\log(e+BT/\epsilon)\right)
\end{align*}
which matches prior constrained optimistic guarantees up to sub-asymptotic log factors while being robust to poorly chosen $h_t$.

\section{Many Hints At Once}\label{sec:manyhints}
The classical optimistic online learning setup considers a single hint $h_t$ provided at each round. However, one could also imagine a scenario in which \emph{multiple} hints $h_{t,i},\dots,h_{t,i}$ are provided in each round. Our reduction allows us to handle this case seamlessly by combining the best-of-all-words analysis of Theorem \ref{thm:combine} with Algorithm \ref{alg:opt}. In a nutshell, we consider updates of the form $x_t -\sum_{i=1}^k y_{t,i}h_{t,i}$ and use $k$ independent 1-dimensional optimizers to optimize each $y_{t,i}$. This roughly corresponds to applying Theorem \ref{thm:combine} to the problem of choosing the best hints, and so we suffer only an additive penalty of $\epsilon k$ to compete with the best hint sequence. The resulting pseudocode is in Algorithm \ref{alg:optmanyhints}, which we analyze in Theorem \ref{thm:optmanyhints}

\begin{algorithm}
   \caption{Optimism with Many Hints}
   \label{alg:optmanyhints}
\begin{algorithmic}
   \STATE {\bfseries Input:} Online learning algorithm $\ol$ with domain $W$ and $\bol$ with domain $\R$.
   \STATE Initialize $k$ copies of $\bol$, $\bol_1,\dots,\bol_k$.
   \FOR{$t=1$ {\bfseries to} $T$}
   \STATE Get $x_t$ from $\ol$ and $y_{t,i}$ from $\bol_i$ for $i\in[1,k]$.
   \STATE Get hints $h_{t,1},\dots,h_{t,k}$.
   \STATE Play $w_t=x_t - \sum_{i=1}^k y_{t,i}h_{t,i}$, receive loss $g_t$.
   \STATE Send $g_t$ to $\ol$ as the $t$th loss.
   \STATE Send $-\langle g_{t,i}, h_t\rangle$ to $\bol_i$ as the $t$th loss for all $i$.
   \ENDFOR
\end{algorithmic}
\end{algorithm}

This reduction obtains the guarantee
\begin{restatable}{Theorem}{thmoptmanyhints}\label{thm:optmanyhints}
Under the same assumptions as Theorem \ref{thm:opt},
% Suppose $\ol$ guarantees regret
% \begin{align*}
%     R^{\ol}_T(\w)\le A_T(\w)+B_T(\w)\sqrt{\sum_{t=1}^T \|g_t\|^2}
% \end{align*}
% on gradients $g_t$ and suppose $\bol$ guarantees regret
% \begin{align*}
%     R^{\bol}_T(\w)\le \epsilon+|\w|C\log(1+|\w|T^c/\epsilon)+|\w|D\sqrt{\sum_{t=1}^T z_t^2\log(1+|\w|T^c/\epsilon)}
% \end{align*}
% on gradients $z_t$, where $A_T$ and $B_T$ are arbitrary non-negative functions and $C$, $c$ and $D$ and $\epsilon$ are arbitrary nonnegative constants.
Algorithm \ref{alg:opt} guarantees regret
\begin{align*}
R_T(\w)&\le B_T(u)\sqrt{\left[(2C+D^2)\log(e+B_T(u)T^c/\epsilon)+\sum_{t=1}^T \|h_{t,i}-g_t\|^2-\|h_{t,i}\|^2\right]_1}\\
    &\quad\quad\quad+DB_T(u)\sqrt{\log(e+B_T(u)T^c/\epsilon)} + A_T(u)+k\epsilon
\end{align*}
for all $i$, where $[X]_1$ denotes $\max(X,1)$. Further, Algorithm \ref{alg:opt} \emph{simultaneously} guarantees regret
\begin{align*}
    R_T(\w) &\le k\epsilon + R^{\ol}_T(\w)
\end{align*}
\end{restatable}

\section{Best Fixed Hint}\label{sec:fixed}

So far we have discussed how to use hints $h_t$ effectively, but given no consideration to where the hints come from. In many cases, there is no external oracle providing hints and so they must be constructed from other information. One popular choice is $h_t=g_{t-1}$. This yields bounds that depend on $\sum_{t=1}^T \|g_t-g_{t-1}\|^2$ and so obtain low regret when the gradients are ``slowly varying''. Another approach suggested by \citep{hazan2010extracting} yields regret bounds that depend on $\sum_{t=1}^T \|g_t -\overline{g}\|^2$ where $\overline{g}=\tfrac{1}{T} \sum_{t=1}^T g_t$ - which is an optimistic regret bound using the best \emph{fixed} hint. In this section, we suggest a simple scheme that generates hints that perform as well as this latter bound, which somewhat streamlines the analysis of \citep{hazan2010extracting}.  By utilizing Theorem \ref{thm:optmanyhints}, we can obtain both bounds at the same time.

The technique is quite simple: we use an online learning algorithm to choose $h_t$. Define $\ell_t(h) = \|g_t-h\|^2$, then we have
\begin{align*}
    \sum_{t=1}^T \|g_t-h_t\|^2\le \sum_{t=1}^T \ell_t(h_t)-\ell_t(\overline{h})+\sum_{t=1}^T \|g_t-\overline{h}\|^2
\end{align*}
for any arbitrary $\overline{h}$. Further $\sum_{t=1}^T \ell_t(h_t)-\ell_t(\overline{h})$ is simply the regret of an online algorithm that plays $h_t$ in response to losses $\ell_t$. Conveniently, $\ell_t(h)$ is strongly-convex and so if we use the Follow-The-Leader algorithm to pick $h_t$ (which corresponds to using the running-average $h_t=\frac{1}{t-1}\sum_{i=1}^{t-1}g_i$), we obtain \citep{mcmahan2014survey}:
\begin{align*}
    \sum_{t=1}^T \ell_t(h_t)-\ell_t(\overline{h})\le O(\log(T))
\end{align*}
Plugging this into the regret bound of Theorem \ref{thm:opt}, we have regret
\begin{align}
    R_T(\w)&\le \tilde O\left (\epsilon + \|\w\|\sqrt{\sum_{t=1}^T \|g_t-\overline{g}\|^2}\right )\label{eqn:fixed}
\end{align}
Up to log factors, this represents a generic improvement in adaptivity over the standard regret bound that depends on $\sum_{t=1}^T \|g_t\|^2$, and generalizes the regret bound of \citep{hazan2010extracting} to unconstrained domains. Further, we remark in Appendix \ref{sec:concentration} that the existence of this algorithm actually provides a simple proof of an empirical Bernstein bound in Hilbert spaces.

The above technique can actually be improved in the unconstrained setting. Since our unconstrained optimistic guarantee depends on $\sum_{t=1}^T \|g_t-h_t\|-\|h_t\|^2=\sum_{t=1}^T \langle g_t,g_t-2h_t\rangle$ rather than $\sum_{t=1}^T \|g_t-h_t\|^2$, we can set $\ell_t(h) = \langle g_t, g_t-2h\rangle$ to obtain an even tighter bound. Notice now that $\ell_t$ is no longer strongly-convex, but it is still convex.  Thus we can use an adaptive gradient descent algorithm with domain $\{\|h\|\le 1\}$ (e.g. Adagrad) \citep{duchi10adagrad, mcmahan2010adaptive} to obtain:
\begin{align*}
    \sum_{t=1}^T \ell_t(h_t)-\ell_t(\overline{h}) \le O\left(\sqrt{\sum_{t=1}^T \|g_t\|^2}\right)
\end{align*}
In this case the optimal value of $\overline{h}$ is $-\tfrac{\sum_{t=1}^T g_t}{\|\sum_{t=1}^T g_t\|}$, so that we have
\begin{align*}
    \sum_{t=1}^T \|g_t -h_t\|^2 - \|h_t\|^2&\le \sum_{t=1}^T \|g_t-\overline{h}\|^2-\|\overline{h}\|^2+\ell_t(h_t)-\ell_t(\overline{h})\\
    &\le \sum_{t=1}^T \|g_t\|^2 - 2\left\|\sum_{t=1}^T g_t\right\|+O\left (\sqrt{\sum_{t=1}^T \|g_t\|^2}\right)
\end{align*}
If we then apply the optimistic bound of Theorem \ref{thm:opt}, we have
\begin{align*}
    R_T(\w) &\le \tilde O\left(\epsilon + \|\w\|\sqrt{\sum_{t=1}^T \|g_t\|^2-2\left\|\sum_{t=1}^T  g_t\right\| + \sqrt{\sum_{t=1}^T \|g_t\|^2}}\right)
\end{align*}

\section{Conclusion}
We introduced the simple strategy of adding iterates as a method for obtaining best-of-all-worlds style regret guarantees in parameter-free online learning. Further, a variation on this technique yields \emph{optimistic} regret bounds. Our optimistic algorithm is a generic reduction that converts any adaptive online learning algorithm into an optimistic algorithm. This extends optimism to unconstrained domains, allows algorithms to use many sequences of hints, and does not degrade performance when the hints are poor. Finally, we provide a simple technique that competes with the best fixed hint, which can be used to provide a simple proof of an empirical Bernstein bound. Intuitively, we achieved optimism by combining an algorithm that had an excellent best-case guarantee but a poor worst-case guarantee with a ``safety-net'' algorithm that had reasonable worst-case guarantees. It is our hope that similar synergies with other algorithms will yield further increases in adaptivity.

\small
\bibliographystyle{unsrt}
\bibliography{all}

\appendix

\section{Concentration Inequality}\label{sec:concentration}
In this section we convert the ``best fixed hint'' bound in Section \ref{sec:fixed} into a concentration inequality in Hilbert spaces following the approach of \cite{rakhlin2015equivalence}, who describe an elegant general equivalence between online learning algorithms and concentration inequalities. Although we suspect the constants in our bound can be significantly improved by a more involved direct analysis, we think the simplicity of this argument is interesting in of itself. First we describe the general procedure to turn regret bounds into concentration inequalities. We run an online algorithm with gradients $g_t=X_t-\E[X_t]$ where $X_1,\dots,X_t$ are i.i.d. random variables such that $\|X_t- \E[X_t]\|\le 1$ with probability 1.
Suppose our algorithm guarantees $R_T(0)\le \epsilon$ for some $\epsilon$. Then if we set $\w=-c\frac{\sum_{t=1}^T g_t}{\|\sum_{t=1}^T g_t\|}$ for some $c$, we have:
\begin{align*}
    R_T(\w) = \sum_{t=1}^T \langle g_t, w_t\rangle + c\left \|\sum_{t=1}^T g_t\right\|\\
    \epsilon - \sum_{t=1}^T \langle g_t, w_t\rangle = c\left\|\sum_{t=1}^T g_t\right\| - R_T(u) +\epsilon
    \epsilon = \E\left[c\left\|\sum_{t=1}^T g_t\right\| - R_T(u) +\epsilon\right]
\end{align*}

Further, since $\epsilon-\sum_{t=1}^T \langle g_t, w_t\rangle=\epsilon - R_T(0)$, we have $\epsilon-\sum_{t=1}^T \langle g_t, w_t\rangle\ge 0$ so that by Markov's inequality we can say that with probability at least $1-\delta$,
\begin{align*}
    c\left\|\sum_{t=1}^T g_t\right\| - R_T(\w) +\epsilon&\le \frac{\epsilon}{\delta}
\end{align*}
This, in tandem with appropriate algebra, provides a concentration inequality of roughly $\|\sum_{t=1}^T X_t-\E[\sum_{t=1}^T X_t]\|\le R_T(-\tfrac{\sum_{t=1}^T g_t}{\|\sum_{t=1}^T g_t\|})$. In particular, since $R_T(\w)$ depends on $g_t-\overline{g}$, we have $g_t-\overline{g}=X_t-\overline{X}$ where $\overline{X} = \tfrac{1}{T}\sum_{t=1}^T X_t$, and so we recover the empirical Bernstein inequality \citep{maurer2009empirical}, generalized to Hilbert spaces:
\begin{restatable}{Theorem}{thmebi}\label{thm:ebi}
Suppose $X_1,\dots,X_T$ are i.i.d. random variables in a Hilbert space such that $\|X_t-\E[X_t]\|\le 1$ with probability 1. Then with probability at least $1-\delta$,
\begin{align*}
    \left\|\sum_{t=1}^T X_t - \E\left[\sum_{t=1}^T X_t\right]\right\|&\le \tilde O\left(1+\sqrt{\sum_{t=1}^T \left\|X_t-\tfrac{\sum_{t=1}^T X_t}{T}\right\|^2}\right)
\end{align*}
\end{restatable}
\begin{proof}
Recall that our proof strategy is to run an online learning algorithm on the gradient sequence $g_t = X_t-\E[X_t]$, which yields 
\begin{align*}
    R_T(\w) = \sum_{t=1}^T \langle g_t, w_t\rangle + c\left \|\sum_{t=1}^T g_t\right\|\\
    \epsilon = \E\left[c\left\|\sum_{t=1}^T g_t\right\| - R_T(u) +\epsilon\right]
\end{align*}
Suppose the online learning algorithm is an algorithm that obtains an optimistic guarantee with a best fixed-hint:
\begin{align*}
    R_T(\w) &\le O\left(\epsilon+ \|\w\|\sqrt{\sum_{t=1}^T  \|g_t - \overline{g}\|^2\log(\|w\|T/\epsilon)} + \|\w\|\log(\|\w\|T/\epsilon)\right)
\end{align*}
where $\overline{g} = \tfrac{1}{T}\sum_{t=1}^T g_t$. Then $R_T(0)\le \epsilon$ with probability 1, so that $\epsilon-\sum_{t=1}^T \langle g_t, w_t\rangle\ge 0$ with probability 1. Therefore by Markov's inequality, with probability at least $1-\delta$, 
\begin{align*}
    c\left\|\sum_{t=1}^T g_t\right\| - R_T(\w) +\epsilon&\le \frac{\epsilon}{\delta}\\
    \left\|\sum_{t=1}^T g_t\right\| &\le \frac{R_T(\w)}{c} + \frac{\epsilon(1+\delta)}{c\delta}\\
    \left\|\sum_{t=1}^T X_t - \E[\sum_{t=1}^T X_t]\right\| &\le \inf_{c} \frac{R_T(\w)-\epsilon}{c} + \frac{\epsilon}{c\delta}
\end{align*}
Now setting $\epsilon =\delta$ and $c=1$ we have with probability at least $1-\delta$:
\begin{align*}
    \left\|\sum_{t=1}^T X_t - \E[\sum_{t=1}^T X_t]\right\| &\le R_T(\w)-\epsilon +1
\end{align*}

Plugging in our optimistic regret bound we have with probability at least $1-\delta$:
\begin{align*}
    \left\|\sum_{t=1}^T X_t - \E[\sum_{t=1}^T X_t]\right\|&\le  O\left(1+ \|\w\|\sqrt{\sum_{t=1}^T  \|g_t - \overline{g}\|^2\log(T/\delta)} + \log(\|\w\|T/\epsilon)\right)\\
    &=O\left(1+ \sqrt{\sum_{t=1}^T \|X_t - \overline{X}\|^2\log(T/\delta)} + \log(T/\delta)\right)
\end{align*}
where we have observed that $\|u\|=c$ and $X_t-\tfrac{\sum_{t=1}^T X_t}{T}=g_t-\overline{g}$.
\end{proof}

\section{Technical Lemmas}
In this section we prove the Lemmas used in the main text. First, the following Lemma shows that we can discretize the space of $p$-norms:
\begin{restatable}{Lemma}{thmpgrid}\label{thm:pgrid}
Let $q_0=2$ and $\frac{1}{q_{i}} = \frac{1}{q_{i-1}} -\frac{1}{\log(d)}$ for all $i\in\{1,\dots,\lfloor \log(d)/2\rfloor\}$. Let $p_i$ be defined by $\frac{1}{p_i}+\frac{1}{q_i}=1$. Then for any $p\in[1,2]$, there exists $i$ such that $p_i\ge p$, $\|x\|_{p_i}\le \|x\|_p$ and $\|x\|_{q_i}\le e\|x\|_q$ for all $x$, where $\frac{1}{p}+\frac{1}{q}=1$.
\end{restatable}
\begin{proof}
First, we claim that $\|x\|_{q'}\le d^{1/q' - 1/q}\|x\|_q$ for any $q'\le q$. To see this, observe that without loss of generality we may set $\|x\|^q_q=1$, and attempt to maximize $\|x\|_{q'}^{q'}$ subject to the constraint $\|x\|_q=1$. Then by application of LaGrange multipliers, we have
\begin{align*}
    q'x_i^{q'-1} = \lambda q x_i^{q-1}
\end{align*}
for all $i$. From this we see that any non-zero $x_i$s must all be equal to each other. Let $n$ be the number of non-zero $x_i$s, and let $z$ be their common value. Then we wish to maximize $nz^{q'}$ subject to $nz^q=1$. This yields $nz^q=n^{1-q'/q}$. This clearly grows with $n$, which can be at most $d$. Thus we see $\|x\|_{q'}^{q'} \le d^{1-q'/q}$, which implies $\|x\|_{q'}\le d^{1/q'-1/q}$ as desired.

Now we can move on to prove the Lemma. Let $i$ be the largest value such that $q_i\le q$. Then by the recursive definition of $q_i$, we must have $\frac{1}{q_i}-\frac{1}{q}\le \frac{1}{\log(d)}$ so that $\|x\|_{q_i}\le d^{1/\log(d)}\|x\|_q=e\|x\|_q$. Further, since $q_i\le q$, $p_i\ge p$ so that $\|x\|_{p_i}\le \|x\|_p$.
\end{proof}

The following Lemma is used to optimize $y$ in the proof of Theorem \ref{thm:opt}:

\begin{restatable}{Lemma}{thmbalancelogs}\label{thm:balancelogs}
Suppose $A$, $B$, $C$, $D$, $E$ are non-negative constants. Then
\begin{align*}
    &\inf_{y\ge 0}\left[y(A+B\log(e+Cy)) + \frac{D^2}{y} + E\sqrt{\log(e+Cy)}\right]\\
    &\quad\quad\le 2D\sqrt{[A+B\log(e+CD)]_1}+E\sqrt{\log(e+CD)}
\end{align*}
where $[X]_1 = \max(X,1)$.
\end{restatable}
\begin{proof}
We just guess a value for $y$:
\begin{align*}
    y = \frac{ D}{\sqrt{[A+B\log(e+CD)]_1}}
\end{align*}
Then the result follows from the fact that $\log$ is an increasing function and $y\le D$.
\end{proof}

This final Lemma is allows us to compute the modified hint values $\tilde h_t$ needed to convert our unconstrained optimistic algorithm into a constrained algorithm.
\begin{restatable}{Lemma}{thmprojgrad}\label{thm:projgrad}
Let $W$ be a convex domain in a Hilbert space. Let $x\notin W$. Then for any $t\in[0, S(x))$, we have $\delta \in S(x-t\delta)$ for all $\delta\in\partial S(x)$. Further, we have $a\delta \in \partial S(x-S(x)\delta)$ for all $a\in[0,1]$.
\end{restatable}
\begin{proof}
First, observe that from Proposition 1 and Theorem 4 from \citep{cutkosky2018black} we have $S$ is 1-Lipschitz and $\partial S(x)=\left\{\frac{x-\Pi(x)}{\|x-\Pi(x)\|}\right\}$. Therefore $\delta = \frac{x-\Pi(x)}{\|x-\Pi(x)\|}$ and so $S(x)-t\le S(x-t\delta)\le \|x-t\delta - \Pi(x)\|=S(x)-t$, where the first inequality is from Lipschitzness and the last equality from definition of $\delta$. Therefore $\Pi(x)=\Pi(x-t\delta)$, and so the first part of the Lemma follows from Theorem 4 of \citep{cutkosky2018black}. For the second part, we observe that $0\in \partial S(x-S(x)\delta)$ because $x-S(x)\delta\in W$. Therefore $a\delta + (1-a)0 = a\delta$ in $\partial S(x-\delta)$, proving the second part of the Lemma.
\end{proof}

Lemma \ref{thm:htilde} is now an immediate corollary of Lemma \ref{thm:projgrad}.

\section{Proof of Theorem \ref{thm:constrained}}
We restate the Theorem below for reference:
\thmconstrained*
\begin{proof}
Define $\ell_t(w) = \frac{1}{2}\left(\langle g_t, w\rangle + \|g\|S(w)\right)$. Then by Lemma \ref{thm:htilde}, $z_t\in \partial S(\tilde w_t)$, so that $\tilde g_t=\frac{g_t}{2} + \frac{z_t\|g_t\|}{2}\in \partial \ell_t(\tilde w_t)$. Now we apply the definition of $\ell_t$, $\tilde w_t$ and Cauchy-Schwarz just as in \citep{cutkosky2018black} to obtain $\frac{1}{2}\langle g_t ,w_t -\w\rangle \le \ell_t(\tilde w_t) - \ell_t(\w)\le \langle \tilde g_t, \tilde w_t - \w\rangle$. Thus we may analyze the regret of the $\tilde w_t$s with respect to the $\tilde g_t$s. This is encouraging, because the $\tilde w_t$s are constructed using Algorithm \ref{alg:opt} on hints $\tilde h_t$ and gradients $\tilde g_t$.

Now we continue as in the proof of Theorem \ref{thm:opt}. We write the regret
\begin{align*}
    \frac{1}{2}R_T(\w)&\le \sum_{t=1}^T \langle \tilde g_t, \tilde w_t-\w\rangle\\
    &= \sum_{t=1}^T \langle \tilde g_t, x_t - \w\rangle - \langle \tilde g_t, \tilde h_t\rangle y_t\\
    &\le R^{\ol}_T(\w) + R^{\bol}_T(y) - \sum_{t=1}^T \langle \tilde g_t, \tilde h_t\rangle y
\end{align*}
Now again we can actually immediately see the second part of the Theorem by setting $y=0$ and observing $\|\tilde g_t\|\le \|g_t\|$ for all $t$. For the first part of the Theorem, by exactly the same argument as the proof of Theorem \ref{thm:opt} we have 

\begin{align*}
\frac{1}{2}R_T(\w)&\le B_T(u)\sqrt{\left[(2C+D^2)\log(e+B_T(u)T^c/\epsilon)+\sum_{t=1}^T \|\tilde h_t-\tilde g_t\|^2-\|\tilde h_t\|^2\right]_1}\\
    &\quad\quad\quad+DB_T(u)\sqrt{\log(e+B_T(u)T^c/\epsilon)} + A_T(u)+\epsilon
\end{align*}

Finally, observe that
\begin{align*}
    \|\tilde h_t - \tilde g_t\|&\le \left\|\frac{h_t-g_t}{2} - \frac{(\|h_t\|-\|g_t\|)z_t}{2}\right\|\\
    &\le \frac{\|h_t-g_t\|}{2} + \frac{|\|h_t\|-\|g_t\||}{2}\\
    &\le \|h_t-g_t\|
\end{align*}
where in the second line we observed that $\|z_t\|\le 1$ (because $S$ is 1-Lipschitz), and in the last line we observe that $|\|a\|-\|b\||\le \|a-b\|$ for all $a,b$ by triangle inequality. Putting this together we have
\begin{align*}
\frac{1}{2}R_T(\w)&\le B_T(u)\sqrt{\left[(2C+D^2)\log(e+B_T(u)T^c/\epsilon)+\sum_{t=1}^T \| h_t-g_t\|^2\right]_1}\\
    &\quad\quad\quad+DB_T(u)\sqrt{\log(e+B_T(u)T^c/\epsilon)} + A_T(u)+\epsilon
\end{align*}
as desired.
\end{proof}

\section{Proof of Theorem \ref{thm:optmanyhints}}

We restate the Theorem below for reference:
\thmoptmanyhints*

\begin{proof}
Just as in the proof of Theorem \ref{thm:opt}, we write the regret
\begin{align*}
    R_T(\w)&=\sum_{t=1}^T \langle g_t, w_t-\w\rangle\\
    &= \sum_{t=1}^T \langle g_t, x_t - \w\rangle - \sum_{i=1}^k \langle g_t, h_{t,i}\rangle y_{t,i}\\
    &\le R^{\ol}_T(\w) + \sum_{i=1}^kR^{\bol_i}_T(y_i) - \sum_{i=1}^k\sum_{t=1}^T \langle g_t, h_{t,i}\rangle y_i
\end{align*}
Now again we can actually immediately see the second part of the Theorem by setting $y_i=0$ for all $i$ and observing that $\sum_{i=1}^k R^{\bol_i}(0)\le \epsilon k$. Further, choose any particular index $i$. Then set $y_j=0$ for all $j\ne i$ and we have
\begin{align*}
    R_T(\w)&\le R^{\ol}_T(\w) + (k-1)\epsilon+R^{\bol_i}_T(y_i) - 
    \sum_{t=1}^T \langle g_t, h_{t,i}\rangle y_i
\end{align*}
Now the rest of the proof is identical to that of Theorem \ref{thm:opt}
\end{proof}

\end{document}